\documentclass[accepted]{elsarticle}

\usepackage{microtype}
\usepackage{graphicx}
\usepackage{subcaption}
\usepackage{booktabs} 

\usepackage[utf8]{inputenc}
\usepackage{amsmath, amssymb, amsthm, amsfonts}
\usepackage{graphicx}
\usepackage[hidelinks]{hyperref}
\usepackage{braket}
\usepackage[dvipsnames]{xcolor}

\usepackage{bbm}

\usepackage{mathrsfs}

\makeatletter
\newenvironment{breakablealgorithm}
  {
   \begin{center}
     \refstepcounter{algorithm}
     \hrule height.8pt depth0pt \kern2pt
     \renewcommand{\caption}[2][\relax]{
       {\raggedright\textbf{\fname@algorithm~\thealgorithm} ##2\par}%
       \ifx\relax##1\relax 
         \addcontentsline{loa}{algorithm}{\protect\numberline{\thealgorithm}##2}%
       \else 
         \addcontentsline{loa}{algorithm}{\protect\numberline{\thealgorithm}##1}%
       \fi
       \kern2pt\hrule\kern2pt
     }
  }{
     \kern2pt\hrule\relax
   \end{center}
  }
\makeatother

\newcommand\Lt{\mathscr L}

\newcommand\EE{\mathbb E}

\newcommand\ol[2]{\langle#1,#2\rangle}

\newcommand\localE{\mathcal{E}}
\newcommand\lE\localE

\newcommand\ifrac[2]{#1/#2}

\renewcommand\L{\mathcal L}

\newcommand\functional[2]{\braket{#1,#2}}

\renewcommand\v{\mathbf v}

\newcommand\Gtx{G(\theta,X)}
\usepackage[textsize=tiny]{todonotes}

\usepackage{soul}

\usetikzlibrary{positioning, calc, fit}

\newcommand\wf{{\psi}}

\usepackage{hyperref}

\usepackage{algorithmic}
\usepackage{algorithm}

\usepackage{amsmath}
\usepackage{amssymb}
\usepackage{mathtools}
\usepackage{amsthm}

\usepackage[capitalize,noabbrev]{cleveref}

\theoremstyle{plain}
\newtheorem{theorem}{Theorem}[section]

\newtheorem{lemma}[theorem]{Lemma}
\newtheorem{corollary}[theorem]{Corollary}

\newtheorem{problem}{Problem}[section]
\theoremstyle{definition}

\newtheorem{assumption}[theorem]{Condition}
\newtheorem{remark}[theorem]{Remark}

\usepackage{environ}
\NewEnviron{temp}{}

\begin{document}

\author[1]{Nilin Abrahamsen\corref{cor1}}
\ead{nilin@berkeley.edu}

\author[2]{Zhiyan Ding}
\ead{zding.m@math.berkeley.edu}

\author[2]{Gil Goldshlager}
\ead{ggoldshlager@gmail.com}

\author[2,3]{Lin Lin}
\ead{linlin@math.berkeley.edu}

\cortext[cor1]{Corresponding author}

\address[1]{
The Simons Institute for the Theory of Computing,
Berkeley, CA 94720 USA
}
\address[2]{
Department of Mathematics,
University of California, Berkeley, CA 94720 USA
}
\address[3]{
Applied Mathematics and Computational Research Division, Lawrence Berkeley National Laboratory, Berkeley, CA 94720, USA
}

\title{Convergence of variational Monte Carlo simulation and scale-invariant pre-training}

\begin{abstract}
    We provide theoretical convergence bounds for the variational Monte Carlo (VMC) method as applied to optimize neural network wave functions for the electronic structure problem. We study both the energy minimization phase and the supervised pre-training phase that is commonly used prior to energy minimization. For the energy minimization phase, the standard algorithm is scale-invariant by design, and we provide a proof of convergence for this algorithm without modifications. The pre-training stage typically does not feature such scale-invariance. We propose using a scale-invariant loss for the pretraining phase and demonstrate empirically that it leads to faster pre-training.
\end{abstract}
\maketitle

\section{Introduction}

A central goal of computational physics and chemistry is to find a \emph{ground state} which minimizes the energy of a system. Electronic structure theory studies the behavior of electrons that govern the chemical properties of molecules and materials. The energy in the electronic structure is a functional defined on the set of quantum wave functions $\psi$, which are normalized $\Lt^2$-integrable functions determined up to an arbitrary scalar multiplication factor, known as a \emph{phase}. Equivalently, the wave function can be viewed as an element of the projective space rather than a normalized function. 

The \emph{variational Monte Carlo} (VMC) algorithm \cite{Foulkes2001, Toulouse2016, Becca2017} is a powerful method for optimizing the energy of a parameterized wave function through stochastic gradient estimates. The method relies on Monte Carlo samples from the probability distribution defined by the Born rule, which views the normalized squared wave function as a probability density function.

Historically, the number of parameters used in VMC simulations was relatively small, and methods such as stochastic reconfiguration \cite{Sorella98, Sorella2001} and the linear method \cite{Nightingale2001, toulouse2007} were preferred for parameter optimization. Recently, neural networks have been used to parameterize the wave function in an approach known as \emph{neural quantum states} \cite{Carleo2017}. Subsequent work has shown that VMC with neural quantum states can match or exceed state-of-the-art high-accuracy quantum simulation methods for a variety of electronic structure problems, including molecules in both first quantization \cite{Han2019, Hermann2020, pfau2020ab, Hermann_review} and second  quantization \cite{Choo2020, barrett2022autoregressive}, electron gas \cite{Cassella2022, pescia2023messagepassing, wilson2023neural}, and solids \cite{li2022ab}.

Due to the complexity of quantum wave functions parameterized by neural networks, it is in general not possible to explicitly normalize the parameterized wave function. One of the essential properties of the VMC procedure is thus its ability to use an unnormalized function $\psi_\theta$ to represent the corresponding element in the projective space.  This is achieved by pulling back the energy functional to be minimized from the projective space to a \emph{scale-invariant} functional on the space of unnormalized wave functions. This scale-invariant functional is the same as the Rayleigh quotient $\bra{\psi}H\ket{\psi}/\braket{\psi|\psi}$ where the Hermitian operator $H$ is known as the Hamiltonian. The scale-invariance of the energy functional is reflected in the VMC algorithm where a scale-invariant \emph{local energy} function corresponding to the wave function is evaluated at MCMC-sampled electron configurations. For the ground state problem considered here it suffices to restrict to real-valued wave functions, so we will take all scalars to be real. 

The wave function in the VMC algorithm is often initialized using a supervised pre-training stage \cite{pfau2020ab} where it is fitted to a less expressive approximation to the ground state, for example given as a Slater determinant. In contrast to the main VMC algorithm which is trained with an inherently scale-invariant objective, this supervised pre-training typically does not encode the scale-invariant property \cite{pfau2020ab, vonglehn2023selfattention}. We propose using a scale-invariant loss function for this supervised stage and demonstrate numerically that this modification leads to faster convergence towards the target in the supervised pre-training setting. We further give a theoretical convergence bound for the proposed scale-invariant supervised optimization algorithm. As an important ingredient in this analysis, we introduce the concept of a \emph{directionally unbiased} stochastic estimator for the gradient.

\subsection{Related works}

Using scale-invariance in the parameter vector to stabilize SGD training is a well-studied technique. Many structures and methods have been proposed along these lines, such as normalized NN~\cite{Ioffe_2015,Wu_2018,Ba_2016}, $\mathcal{G}$-SGD~\cite{meng_mathcalg-sgd_2021}, 
SGD+WD~\cite{van_2017,arora2018theoretical,li_robust_2022,Wan_2021}, and projected/normalized (S)GD~\cite{Saliman_2016,kodryan_training_2022}. 
In these works, the parameter $\v$ is stored explicitly and the loss function $\L$ is scale-invariant in the parameter $\v$, which means $\L(\lambda\v)=\v$ for any $\lambda>0$. 

The setting of neural wave functions differs from typical scale-invariant loss functions in that the scale-invariant functional $\L$ is applied after a nonlinear parameterization $\theta\mapsto\psi_\theta$, and the function $\psi_\theta$ can be queried but not accessed as an explicit vector. The composed loss function $L(\theta)=\L(\psi_\theta)$ is not scale-invariant in its parameters. We can obtain a relation between the two settings by applying the chain rule involving a functional derivative, but the resulting expression involves factors that we need to estimate by sampling, in contrast with the typical setting of scale-invariance in the parameter space.

\subsubsection{SGD on Riemannian manifolds}\label{re:mnf}
The training of a scale-invariant objective can be seen as a training process on projective space, which is a special case of learning on Riemannian manifolds. The convergence of SGD on Riemannian manifolds is well-studied; for example \cite{bonnabel_stochastic_2013} shows the asymptotic convergence of SGD on Riemannian manifolds. Further variations of the SGD algorithm on Riemannian manifolds have been proposed and studied, such as SVRG on Riemannian manifolds~ \cite{Zhang_2016}, \cite{Sato_2019} and ASGD on Riemannian manifolds~\cite{tripuraneni_averaging_2018}. To the best of our knowledge, all previous works ensure the parameter stays on the manifold by either taking stochastic gradient steps on the manifold or projecting them back onto it. 
This is in contrast to our setting where no projection onto a submanifold is needed.

\subsubsection{Concurrent theoretical analysis}
The concurrent work of~\cite{Wen_2023} proves a similar convergence result for VMC and also takes into account the effect of Markov chain Monte Carlo sampling. However, this work does not consider the setting of supervised pre-training.

\section{Variational Monte Carlo}
A quantum wave function is a square-integrable function $\psi$ on a space of configurations $\Omega$. A typical configuration is $\Omega=\mathbb R^{3N}$ for a system of $N$ electrons\footnote{The spin degrees of freedom are considered as classical and are expressed in the symmetry type of $\psi_\theta$.}. The term ``square-integrable function'' refers to a function $\psi$ with a finite $\Lt^2$-norm $\|\psi\|=\sqrt{\int_{\Omega}|\psi|^2\mathrm{d}x}<\infty$ with respect to the Lebesgue measure. For any scalar $\lambda\in\mathbb{R}$ and $\lambda\neq0$, $\lambda\psi$ represents the same physical state as $\psi$.
In the variational Monte Carlo algorithm, $\psi_\theta$ is an Ansatz parameterized by $\theta$, and the objective is to find the parameter vector $\theta$ such that $\psi_\theta$ minimizes the energy
\[\mathcal L(\psi)=\frac{\bra{\psi}H\ket{\psi}}{\braket{\psi|\psi}}.\]
The \emph{Born rule} associates to a quantum wave function $\psi$ a probability distribution on $\Omega$ with density $p_\wf$ given by
\begin{equation}p_\wf(x)=|\wf(x)|^2/\|\wf\|^2\label{pdef}\,.\end{equation}
The norm $\|\wf\|$ is unknown to the optimization algorithm, and samples
$\{X_i\} \sim p_\wf$ are generated using Markov chain Monte Carlo (MCMC). We will assume that these samples are independent and sampled from the exact distribution $p_\wf$. In practice, the independence assumption is tested by evaluating the autocorrelation of the MCMC samples.  

The energy of $\wf$ can then be expressed as an expectation with respect to $p_\wf$:
$
\L(\wf)=\EE_{X\sim p_{\wf}}\localE_\wf(X),
$
where
\begin{equation}\label{eqn:deflocalE}
\localE_\wf(x)=\frac{(H\wf)(x)}{\wf(x)}
\end{equation}
is called the \emph{local energy}.  We note that the local energy is not well-defined in the case where $\psi(x)=0$, but this is not a practical issue since $p_\psi(x)=0$ by definition at such points and thus the expectation value is still well-defined. 

We use $L(\theta)=\L(\psi_\theta)$
to denote the loss as a function of the parameters.

\begin{problem}[VMC]\label{vmcprob}
Minimize
    \begin{equation}\label{eqn:VMC_L}
        L(\theta)=\EE_{X\sim p_\theta}\localE_\theta(X),
    \end{equation}
    over parameter vectors $\theta\in\mathbb{R}^d$, where $\localE_\theta(x)=H\psi_\theta(x)/\psi_\theta(x)$ and $p_{\theta}(x)=|\psi_{\theta}(x)|^2/\|\psi_\theta\|^2$. We assume access to samples from $\{X_i\}\sim p_\theta$ and query access to $\psi_\theta$, $\nabla_\theta \psi_\theta$, and $\localE_\theta$ which are assumed to be $C^\infty$ with respect to $\theta$ for any $x\in\Omega$. 
\end{problem}

When $\|\psi_\theta\|<\infty$, the gradient of the VMC energy functional takes the form (see e.g., equation (9) of \cite{pfau2020ab})
\begin{equation}
\label{eqn:gradient_L_VMC}
\nabla_{\theta} L(\theta)=2 \mathbb{E}_{X\sim p_\theta}\Big[\left(\mathcal{E}_{\theta}(X)-L(\theta)\right) \nabla_{\theta}\log|\psi_{\theta}(X)|\Big].
\end{equation}
Notably, the change in local energy does not contribute to this formula, that is, $\EE_{X\sim p_\theta}[\nabla_\theta\mathcal E_\theta(X)]$=0. Instead, the gradient of $L$ is entirely due to the 
the perturbation of the sampling distribution as $\theta$ varies. For completeness, we include the derivation of \cref{eqn:gradient_L_VMC} in \ref{gradproofs}. The condition that $\psi_\theta$ is $C^\infty$ with respect to $\theta$ is typically satisfied by using the logistic sigmoid activation function.

\begin{temp}
\subsection{Relation to the policy gradient method}
In the reinforcement learning framework \cite{polgrad}, an agent decides on an action $a$ with probability $\pi_\theta(a|s)$ where $s$ is the state of the environment and $\pi_\theta$ is the \emph{policy}, parameterized by $\theta$. Let $J(\pi)$ be the expected reward $R(\tau)$ of a trajectory sampled under the policy $\pi$. The \emph{policy gradient theorem} \cite{polgrad} exhibits an expression for the gradient of the reward:
\begin{theorem}[\cite{polgrad}]
The gradient of the expected reward is
\begin{equation}\label{eq:polgrad}
    \nabla_\theta J(\pi_\theta)=\sum_{t=1}^T\EE_\pi\big[R(\tau)\nabla_\theta\log\pi_\theta(a_t|s_t)\big],
\end{equation}
where $R(\tau)$ is the reward of trajectory $\tau$, and $s_t,a_t$ are the states and actions in the trajectory.
\end{theorem}
To see the connection to the gradient estimate in VMC, it suffices to consider the policy gradient formula with a single timestep. Then we can ignore the state $s$ and the expected reward becomes $J(\pi)=\EE[Q(a)]$ where $Q(a)$ is the value of action $a$. \cref{eq:polgrad} then states that
\[\nabla_\theta J(\pi_\theta)=\EE\big[Q(a)\nabla_\theta\log\pi_\theta(a)\big].\]
\end{temp}

\section{Supervised Pre-training}\label{sec:super_l}

In order to stabilize the VMC training, the neural quantum state $\psi_\theta$ in FermiNet \cite{pfau2020ab} and subsequent works \cite{Spencer_2020, Gerard2022, Glehn2022} is initialized by fitting the Ansatz to an initial guess $\psi$ at an approximate ground state \cite{pfau2020ab,vonglehn2023selfattention}. This target initial state can be taken to be a Slater determinant optimized using the self-consistent field (SCF) method \cite{szabo1996modern}.
We therefore consider the problem of minimizing the distance from the line $\{\lambda \psi|\lambda\in\mathbb R\}$ to a target function $\varphi\in\Lt^2(\Omega)$. 

%

 That is, 
\begin{equation}\label{supervisedC}
\mathcal{L}(\psi)=\min_{\lambda\in \mathbb R}\left\|\lambda \psi-\varphi\right\|_\rho^2\,,
\end{equation}
where $\left\|\phi\right\|_\rho=\sqrt{\int_{\Omega}|\phi(x)|^2\mathrm{d}\rho(x)}$ is the norm in $\Lt^2$-sense on a probability space $(\Omega,\rho)$ with probability measure $\rho$. 
In practice the pre-training may fit intermediate vector-valued features such as orbitals~\cite{vonglehn2023selfattention} instead of the scalar-valued wave function. As we note in \cref{sec:pretrain-numerical}, this is formally equivalent to fitting a scalar wave function. Moreover, as we show in \cref{sec:pretrain-numerical}, exploiting scale-invariance when fitting the orbitals can improve the convergence of the wave function itself as measured by \cref{supervisedC}.

Here, $\rho$ is typically the Lebesgue measure, but in \cref{sec:pretrain-numerical} we will also consider a case where it is defined by the target function. 
This loss function is scale-invariant by construction and has a closed-form expression
\begin{equation}\label{sinsq}
\mathcal{L}(\psi)=\|\varphi\|_\rho^2-\big(|\ol{\varphi}{\psi}_\rho|/\|\psi\|_\rho\big)^2\,.
\end{equation} 

Note that for normalized target $\varphi$, the second term of \cref{sinsq} is the squared sine of the angle between the vectors, $\mathcal L(\psi)=\sin^2\angle(\varphi,\psi)$.
Minimizing $\mathcal{L}(\psi)$ is equivalent to minimizing $\L(\psi)=-\ol{\varphi}{\psi}_\rho/\|\psi\|_\rho$. Again we write the loss as $L(\theta)=\L(\psi_\theta)$ as a function of the variational parameters:
\begin{problem}[Supervised learning pre-training]\label{pro:slos}
Given a target function $\varphi\in\Lt^2(\Omega,\rho)$, access to samples $\{X_i\}\sim\rho$ and query access to $\varphi$, $\psi_\theta$, and $\nabla_\theta \psi_\theta$, 
minimize
\begin{equation}\label{eqn:Ltheta}
L(\theta)=-\frac{\ol{\varphi}{\psi_\theta}_\rho}{\|\psi_\theta\|_\rho}\,
\end{equation}
over parameter vectors $\theta\in\mathbb{R}^d$.

\end{problem}

\subsection{Directionally unbiased gradient estimator}
By applying the chain rule we can write the gradient of $L(\theta)$ as
\begin{align}\label{chainrule}
\nabla_\theta L(\theta)
&=[\partial_\psi\L](\psi_\theta)\cdot \nabla_\theta\psi_\theta,
\end{align}
where $\partial_\psi\mathcal L(\psi_\theta)$ is the functional derivative of the scale-invariant loss. This can be evaluated analogously to the gradient derived in \cite{Saliman_2016} for a scale-invariant loss function. The resulting gradient expression is
\begin{equation}\label{gradient_L_supervised}
\begin{aligned}
&\nabla_\theta L(\theta)=-\frac{\functional{\varphi}{\nabla_\theta \psi_\theta}_\rho}{\|\psi_\theta\|_\rho}
+
\frac{\functional{\varphi}{\psi_\theta}_\rho
\functional{\psi_\theta}{\nabla_\theta \psi_\theta}_\rho}
{\|\psi_\theta\|_\rho^3}\,.
\end{aligned}
\end{equation}

We cannot compute the infinite-dimensional $\Lt^2$ norms 
in the denominator of \cref{gradient_L_supervised} and do not have access to an unbiased estimator for \cref{gradient_L_supervised}. However, in \cref{lem:unbiased_super}, we propose a \emph{directionally unbiased} gradient estimator $G$, which means that $\EE[G]$ is in the positive span of $\nabla_\theta L$. Additionally, we will demonstrate that the directionally unbiased gradient estimator is sufficient for achieving rapid convergence of SGD as long as the norm $\|\psi_\theta\|_\rho$ can be approximately estimated, see detail in \cref{cor:super}.

\section{Theoretical results}
\subsection{Convergence result for VMC}
In the VMC setting, we assume that the MCMC subroutine is able to sample exactly from the distribution $p_\theta(x)=|\psi_\theta(x)|^2/\|\psi_\theta\|^2$. This sampling oracle makes it possible to construct a cheap and unbiased gradient estimation for the energy functional \eqref{eqn:VMC_L}.
For simplicity we use exact real-valued arithmetic.

From the formula \eqref{eqn:gradient_L_VMC} for the gradient, we derive the following unbiased gradient estimator for $\nabla_\theta L(\theta)$:

\begin{lemma}\label{lem:unbiased_gradient} Given $n \geq 2$ i.i.d. samples $\{X_{i}\} \sim p_{\theta}$, let $\hat{L}(\theta)=\frac{1}{n} \sum_{i=1} \mathcal{E}_{\theta}\left(X_{i}\right)$ and define the gradient estimate
\begin{equation}\label{eqn:VMCgradient_estimator}
\Gtx=\frac{2}{n-1} \sum_{i=1}^{n}\left(\mathcal{E}_{\theta}\left(X_{i}\right)-\hat{L}(\theta)\right) \nabla_{\theta} \log |\psi_{\theta}\left(X_{i}\right)|.
\end{equation}
Then $\Gtx$ is an unbiased estimator of the gradient of the population energy.
\end{lemma}
We prove \cref{lem:unbiased_gradient} in \ref{gradproofs}. 

Using the unbiased gradient estimator \eqref{eqn:VMCgradient_estimator} from \cref{lem:unbiased_gradient}, the VMC algorithm can be formulated as a variant of SGD:
\begin{breakablealgorithm}
      \caption{Variational Monte Carlo}
  \label{algo:VMC}
  \begin{algorithmic}[1]
  \STATE \textbf{Preparation:} number of iterations: $M$; learning rate $\eta_m$; initial parameter: $\theta_0$; number of samples each iteration: $n$;
  \STATE \textbf{Running:}
  \STATE $m\gets 0$;
  \WHILE{$m\leq M$}
  \STATE Sample $\{X^m_i\}^n_{i=1}$ independently according to the density $p_{\theta^m}\propto|\psi_{\theta^m}|^2$;
  \STATE $G_m\gets\Gtx=\frac{2}{n-1} \sum_{i=1}^{n}(\mathcal{E}_{\theta}(X_{i})-\hat{L}(\theta)) \nabla_{\theta} \log |\psi_{\theta}\left(X_{i}\right)|$;
  \STATE $\theta_{m+1}\gets \theta_m-\eta_m G_m$;
  \ENDWHILE
    \STATE \textbf{Output:} $\theta_m$;
    \end{algorithmic}
\end{breakablealgorithm}

Given any $\theta$, we define $\|\cdot\|_{q,p_\theta}$ as the $q$-th norm under the measure $ p_\theta(x)dx$. 
The convergence bound for the VMC algorithm assumes a uniform bound on the following quantities:

\begin{assumption}\label{assum_g_VMC} There exists a constant $C_{\wf}$ such that for any $\theta\in\mathbb{R}^d$, 
\begin{itemize}
    \item (Value bound): 
    \begin{equation}\label{eqn_value_gnbound_vmc}
    \|H\psi_\theta/\psi_\theta\|_{4,p_\theta}\leq C_\wf\,.
    \end{equation}
    
    \item (Gradient bound): 
    \begin{equation}\label{eqn_Gradient_gnbound_vmc}
    \begin{aligned}
&\|\nabla_\theta H\psi_\theta/\psi_\theta\|_{2,p_\theta}\leq C_\wf\,,\\
&\|\nabla_\theta \psi_\theta/\psi_\theta\|_{4,p_\theta}\leq C_\wf
\,.
\end{aligned}
    \end{equation}
    \item (Hessian bound): 
    \begin{equation}\label{eqn_Hessian_gnbound_vmc}
    \|\mathsf{H}_\theta\psi_\theta/\psi_\theta\|_{2,p_\theta}\leq C_\wf\,,
    \end{equation}
    where $\mathsf{H}_\theta \psi_\theta$ is the hessian of $\psi_\theta$ in $\theta$.
\end{itemize}

\end{assumption}
\cref{assum_g_VMC} is scale-invariant in $\psi_\theta$, i.e., if $\psi_\theta(x)$ satisfies \cref{assum_g_VMC}, then $\lambda\psi_\theta(x)$ also satisfies the condition with the same constant $\lambda$ for any $\lambda\neq 0$. We establish new bounds on the Lipschitz constant of $\nabla L$ and the variance of the gradient estimator (see \cref{lem:prior_VMC}), subject to the constraints specified in \cref{assum_g_VMC}. These bounds are crucial in demonstrating the convergence of our method.

Now, we are ready to give the convergence result for Algorithm \ref{algo:VMC}:
\begin{theorem}\label{thm:complexity_VMC}
Under \cref{assum_g_VMC} there exists a constant $C>0$ that only depends on $C_\wf$ such that, when $\eta_m<\frac{1}{C}$ for any $m\geq 0$, 
\begin{equation}\label{eqn:gradient_bound_vmc_1}
\sum^M_{m=0}\eta_m\mathbb{E}\left(\left|\nabla_\theta L(\theta_m)\right|^2\right)\leq 2L(\theta_0)+C\sum^M_{m=0}\frac{\eta^2_m}{n}\,.
\end{equation}
for any $M>0$. 
\end{theorem}
The proof of the above theorem is given in \ref{sec:thm:complexity_VMC}. 
The upper bound \eqref{eqn:gradient_bound_vmc_1} is important for us to determine $\eta_m$ so that the parameter $\theta_m$ converges to a first-order stationary point in the expectation sense when $m$ is moderately large. Theorem \ref{thm:complexity_VMC} implies:
\begin{corollary}\label{cor:VMC} Under \cref{assum_g_VMC}:  
\begin{itemize}
\item Choosing $\eta_m=\Theta\left(n\epsilon^2\right)$ and $M=\Theta\left(1/(n\epsilon^4)\right)$,
we have \[\min_{0\leq m\leq M}\mathbb{E}\left(\left|\nabla_\theta L(\theta_m)\right|\right)\leq \epsilon\].

\item Choosing $\eta_m=\Theta\left(\sqrt{\frac{n}{m+1}}\right)$, we have 
\[\min_{0\leq m\leq M}\mathbb{E}\left(\left|\nabla_\theta L(\theta_m)\right|\right)=O\left(\left(nM\right)^{-1/4}\right).\]
\end{itemize}
\end{corollary}

This bound shows that the convergence of Algorithm \ref{algo:VMC} is the same as the convergence rate of SGD to first-order stationary point for nonconvex functions, see~\cite[Theorem 2]{khaled2020better} or \cite[Theorem 5.2.1]{garrigos2024handbook}.

\subsubsection{Scale-invariant supervised pre-training}\label{sec:scale_invariant}

The following lemma shows how we can choose a directionally unbiased gradient estimator. Given samples $\{X_i\}$, write $\ol{\psi}{\varphi}_n=\frac1n\sum_{i=1}^n\psi(X_i)\varphi(X_i)$ and $\|{\psi}\|_n^2=\frac1n\sum_{i=1}^n\psi(X_i)^2$.
\begin{lemma}\label{lem:unbiased_super}
Given $n\ge2$ and i.i.d. samples $X_1,\ldots,X_n\sim\rho$, define coefficients
\[a_j=-\|\psi_{\theta}\|_n^2\:\varphi(X_j)+\ol{\varphi}{\psi_{\theta}}_n\:\psi_{\theta}(X_j)\]
and let
\begin{equation}\label{eqn:G_unbiased}
G=\frac{1}{\|\psi_\theta\|_\rho^3}\frac1{n-1}\sum_{j=1}^n a_j\nabla_\theta \psi_{\theta}(X_j)\,.
\end{equation}
Then $\mathbb{E}_X(G)=\nabla_\theta L(\theta)$. 
\end{lemma}

We give the derivation of \cref{lem:unbiased_super} in \ref{sec:minibath_sup}. 
Given an estimator $\tilde{Z}$ for $\|\psi_\theta\|_\rho$ which is independent of $X_1,\ldots,X_n$, we then choose the following as our directionally unbiased gradient estimator:

\begin{equation}\label{eqn:G_r_theta}
\Gtx=\frac{1}{\tilde{Z}^3}\frac1{n-1}\sum_{j=1}^n a_j\nabla_\theta \psi_{\theta}(X_j)\,,
\end{equation}
where 
$\{a_j\}^n_{j=1}$ are defined in Lemma \ref{lem:unbiased_super}.

\begin{remark}\label{rem:indepdenom}
To obtain the independent norm estimate $\tilde Z$ we can take $2n$ samples $X_1,\ldots,X_{2n}$ at each iteration and let $\tilde Z=(\frac1n\sum_{i=n+1}^{2n}|\psi_\theta(X_i)|^2)^{1/2}$ in \eqref{eqn:G_r_theta}. 
However, our numerics in \cref{sec:pretrain-numerical} suggest it is sufficient in practice to estimate $\|\psi_\theta\|_\rho$ as $\tilde Z=(\frac1n\sum_{i=1}^n|\psi_\theta(X_i)|^2)^{1/2}$ without sampling an additional independent minibatch. 
\end{remark}

\begin{remark}
Some care must be taken when constructing a plug-in estimator for the gradient. For example, given an estimate $\tilde{Z}$ of $\|\psi_\theta\|_\rho$ and samples $X_1,X_2\sim\rho$, \cref{gradient_L_supervised} suggests a plug-in estimate of $\nabla_\theta L_\theta$ given by
\begin{equation}\label{eq:plugin}
-\frac{\varphi(X_1)\nabla_\theta \psi(X_1)}{\tilde{Z}}+\frac{\varphi(X_2)\psi(X_2)\psi(X_1)\nabla_\theta \psi(X_1)}{\tilde{Z}^3}\,.
\end{equation}
However, this estimator is not directionally unbiased and does not achieve the convergence shown in this paper. This is due to the fact that it is \emph{unbalanced} in the sense that the exponent of $\tilde{Z}$ is different for each term. More specifically, taking expectation on \eqref{eq:plugin} gives
\[
-\frac{\ol{\varphi}{\nabla \psi_\theta}}{\tilde{Z}}+\frac{\ol{\varphi}{\psi_\theta}\ol{\psi_\theta}{\nabla \psi_\theta}}{\tilde{Z}^3}\neq \lambda \nabla_\theta L(\theta),\quad \forall \lambda\in\mathbb{R}\,.
\]

\end{remark}

The detailed algorithm is summarized in Algorithm \ref{algo:GE}.
\begin{breakablealgorithm}
      \caption{Stochastic gradient descent algorithm for supervised learning}
  \label{algo:GE}
  \begin{algorithmic}[1]
  \STATE \textbf{Preparation:} $\varphi_n$; $\psi_\theta(x)$; number of iterations: $M$; learning rate $\eta_m$; initial parameter: $\theta_0$; $n$
  \STATE \textbf{Running:}
  \STATE $m\gets 0$;
  \WHILE{$m\leq M$}
  \STATE Sample $\{X^m_i\}^{n}_{i=1}$ independently from $\rho$;
  \STATE $\widetilde{Z}_{m}\gets \text{an estimation of}\ \|\psi_{\theta_m}\|_\rho$; 
  \STATE $a_i\gets-\|\psi_{\theta_m}\|_n^2\:\varphi(X^m_i)+\ol{\varphi}{\psi_{\theta_m}}_n\:\psi_{\theta_m}(X^m_i)$
  \STATE $G_{m}\gets \Gtx=\frac{1}{\widetilde{Z}_{m}^3}\frac1{n-1}\sum_{i=1}^n a_i\nabla_\theta \psi_{\theta_m}(X^m_i)$;
  \STATE $\theta_{m+1}\gets \theta_m-\eta_m G_m$;
  \ENDWHILE
    \STATE \textbf{Output:} $\theta_m$;
    \end{algorithmic}
\end{breakablealgorithm}

The choice of learning rate $\eta_m$ in the pre-training setting and the decay rate of the loss function depends on the ratio $\|\psi_{\theta_m}\|_\rho/\widetilde{Z}$. However, for our results it suffices that this ratio is bounded above and below by fixed constants. As mentioned above we may estimate $\tilde Z=(\frac1n\sum_{i=1}^n|\psi_\theta(X_i)|^2)^{1/2}$ at each iteration using an independent sample 
$X_{n+1},\ldots,X_{2n}\perp\!\!\!\perp X_1,\ldots,X_n$. 
Alternatively we may take inspiration from variance reduction techniques~\cite{SAGA-2013,SAGA-2014,Johnson_Zhang}, to update $\tilde Z$ using a large batch once every $K$ steps, letting $\tilde Z_m=(\frac1K\sum_{i=1}^K |\psi_\theta(X_i)|^2)^{1/2}$ when $m/K\equiv0 \:(\text{mod }K)$ and $\widetilde{Z}_m=\widetilde{Z}_{m-1}$ otherwise.

Our convergence bound for the supervised pre-training assumes uniform bounds similar to \cref{assum_g_VMC}.
\begin{assumption}\label{assum_g_scale_inva} There exists a constant $C_\psi$ such that for any $\theta\in\mathbb{R}^d$
\begin{itemize}
    \item (Value bound): 
    \begin{equation}\label{eqn_value_gnbound}
    \left\|\psi_\theta^2\right\|^{1/2}_\rho/\|\psi_\theta\|_\rho\leq C_\psi\,.
    \end{equation}
    \item (Gradient bound): 
    \begin{equation}\label{eqn_Gradient_gnbound}
    \left\|\left|\nabla_\theta \psi_\theta\right|^2\right\|^{1/2}_\rho/\|\psi_\theta\|_\rho\leq C_\psi\,.
    \end{equation}
    \item (Hessian bound): 
    \begin{equation}\label{eqn_Hessian_gnbound}
    \left\|\left\|\mathsf{H}_\theta \psi_\theta\right
    \|^2_2\right\|^{1/2}_\rho/\|\psi_\theta\|_\rho\leq C_\psi\,,
    \end{equation}
    where $\mathsf{H}_\theta \psi$ is the hessian of $\psi$.
\end{itemize}

\end{assumption}
Under \cref{assum_g_scale_inva}, we establish the Lipschitz property of $\nabla_\theta L$ and provide bounds for the variance of the gradient estimator (see \cref{lem:svsl_2}). These results play a crucial role in demonstrating the convergence of our method.
\cref{assum_g_scale_inva} is scale-invariant in $\psi_\theta$, meaning that the same condition holds with the same constant for any $\lambda \psi_\theta$, where $\lambda\neq 0$.

\begin{theorem}\label{thm:main_result_alg1}
Assume \cref{assum_g_scale_inva}, $|\varphi_n|\leq C_\varphi$, and 
\begin{equation}\label{eqn:distance_bound}
\frac{1}{C_r}\leq 
\frac{\|\psi_{\theta_m}\|_\rho}{\widetilde{Z}_m}
\leq C_r\,,\quad \mathrm{a.s.}
\end{equation}
for all $m\in\mathbb{N}$, where $\theta_m$ comes from Algorithm \ref{algo:GE}. Then there exists a constant $C$ that only depends on $C_r,C_\psi,C_\varphi$ such that, when $\eta_m<\frac{1}{C}$ for any $m\geq 0$, 
\begin{equation}\label{eqn:gradient_psi_bound}
\sum^M_{m=0}\eta_m\mathbb{E}\left(\left|\nabla_\theta L(\theta_m)\right|^2\right)\leq 2C^2_rL(\theta_0)+C\sum^M_{m=0}\frac{\eta^2_m}{n}\,.
\end{equation}
for all $M>0$.
\end{theorem}
The above theorem is proved in \ref{sec:thm:main_result_alg1}. In the proof, we first bound the $\mathbb{E}|G_m|^2$ and show that $\nabla_\theta L$ has a uniform Lipschitz constant. Then, the decay rate of the loss function in each step can be lower bounded by $|\nabla_\theta L|^2$, which finally gives us an upper bound of $|\nabla_\theta L|^2$ as shown in \eqref{eqn:gradient_psi_bound}.

Using Theorem \ref{thm:main_result_alg1} it is straightforward to show the following corollary: 
\begin{corollary}\label{cor:super} Under \cref{assum_g_scale_inva} and \eqref{eqn:distance_bound}, 
\begin{itemize}
\item Choosing $\eta_m=\Theta\left(n\epsilon^2\right)$ and $M=\Theta\left(1/(n\epsilon^4)\right)$,
we have 
\[\min_{0\leq m\leq M}\mathbb{E}\left(\left|\nabla_\theta L(\theta_m)\right|\right)\leq \epsilon.\]

\item Choosing $\eta_m=\Theta\left(\sqrt{\frac{n}{m+1}}\right)$,
we have \[\min_{0\leq m\leq M}\mathbb{E}\left(\left|\nabla_\theta L(\theta_m)\right|\right)=O\left(\left(nM\right)^{-1/4}\right).\]

\end{itemize}
\end{corollary}

We note that this bound matches the standard $O(M^{-1/4})$ convergence rate of SGD to first-order stationary point for non-convex functions.

\section{Numerical results}

\subsection{Pre-training with scale-invariant loss}
\label{sec:pretrain-numerical}

To evaluate the use of a scale-invariant loss during pre-training we modify the training loss used in the pre-training stage of \cite{vonglehn2023selfattention} to be scale-invariant with respect to the trained Ansatz. We tested the scale-invariant loss in a version of the FermiNet code \cite{ferminet_github} where the electron configurations were sampled from the target state as in \cite{vonglehn2023selfattention}. Specifically, the wave function is given by a sum of determinants 
and the target state is a Slater determinant: 
\[\psi_\theta(x)=\sum_{k=1}^d\Big[\det( y^{(k)}_{ij})_{i,j=1}^N\Big],\qquad
\varphi(x)=\det\Big[\big(\phi_j(x_i)\big)_{i,j=1}^N\Big],\]
where $N$ is the number of electrons and the tensor $y$ is the output of the parameterized permutation-equivariant neural network. In our experiment we use the standard FermiNet Ansatz~\cite{pfau2020ab} for the network architecture.
The pre-training of \cite{vonglehn2023selfattention} uses the loss function
\begin{equation}\label{standard}\sum_{k=1}^d\sum_{i,j=1}^N|y_{i,j}^{(k)}-\phi_{j}(x_i)|^2\end{equation}
on the matrices. To test the use of a scale-invariant loss function during training we replace the training loss of \cref{standard} with the sum of column-wise scale-invariant $\sin^2$ distances. We note that in practice we did not find it necessary to estimate the denominator using an independent sample as in \cref{rem:indepdenom}. Note also that we can view the fitting of a (column) vector-valued function $\Omega\ni x\mapsto (y_i)_{i=1}^N$ as fitting a scalar-valued function $y(x,i)$ defined on $\Omega\times\{1,\ldots,N\}$. 

\begin{equation}\label{SI_pretrain}L(\theta)=\sum_{k=1}^d\sum_{j=1}^N\Big(1-\frac{|y_{\cdot,j}^{(k)}\cdot\phi_{j}(\cdot)|^2}{\|y_{\cdot,j}^{(k)}\|^2}\Big),\end{equation}
where $\phi_j(\cdot)=(\phi_j(x_1),\ldots,\phi_j(x_n))^T$, that is,
\begin{equation}
L(\theta)=\sum_{k=1}^d\sum_{j=1}^N\Big(1-\frac{|\sum_{i=1}^Ny_{i,j}^{(k)}\phi_{j}(x_i)|^2}{\sum_{i=1}^N|y_{i,j}^{(k)}|^2}\Big).
\end{equation}

\begin{figure}[H]
\begin{center}
\includegraphics[width=.7\textwidth]{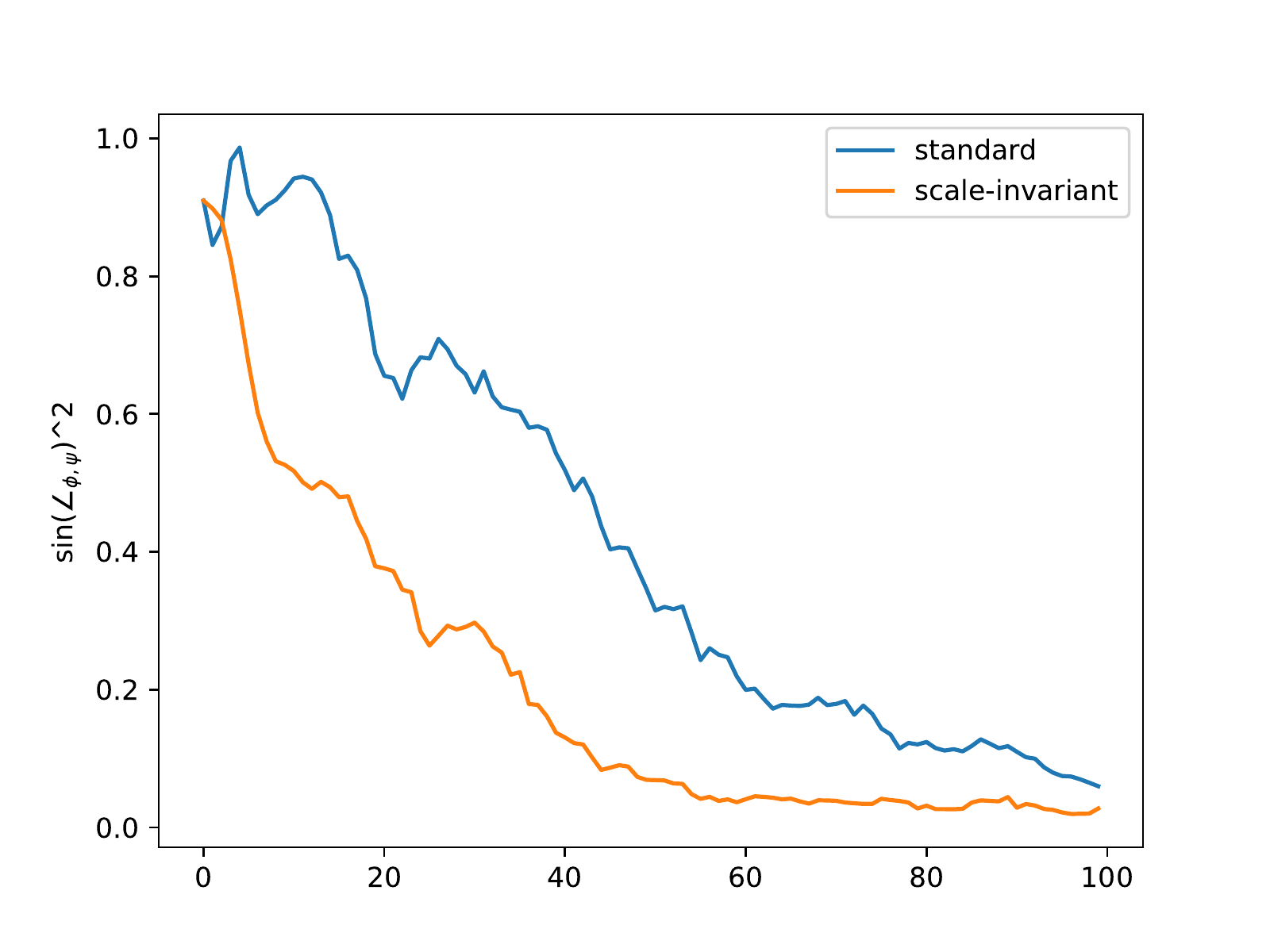}
\caption{Convergence of supervised pre-training using the scale-invariant training loss (orange) vs. the training loss of \cite{vonglehn2023selfattention} (blue). For both optimizers the plotted quantity is the sine of the angle between the target state and the trained state, where the angle is defined in $\Lt^2(\mathbb R^{3n},\rho=|\varphi|^2)$ with respect to the measure induced by the target state density.}
\label{fig:comp}
\end{center}
\end{figure}

To evaluate the performance of the modified pre-training procedure we plot the angle between the target Slater determinant-based wave function and the neural network wave function during each training procedure (\cref{fig:comp}). The system shown is the lithium atom (details in \ref{sec:appnumerics}).

\subsection{Empirical convergence of VMC algorithm}
As an example, we demonstrate  the convergence of the VMC \cref{algo:VMC} on the Hydrogen square (H$_4$) model depicted in \cref{fig:h4conf}. 
This system involves only four particles but is strongly correlated and known to be difficult to simulate accurately. We define $\psi_\theta$ using the FermiNet ansatz \cite{pfau2020ab, Spencer_2020} and run 200,000 training steps using stochastic gradient descent with a learning rate schedule $\eta_m = \frac{0.05}{\sqrt{1 + m/10000}}$. 

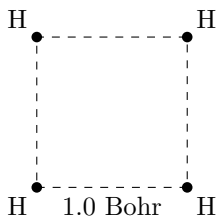
\begin{figure}[H]
\centering
\begin{tikzpicture}

\coordinate (1) at (0,0);
\coordinate (2) at (2,0);
\coordinate (3) at (2,2);
\coordinate (4) at (0,2);
\coordinate (5) at ($(1)!.5!(2)$);

\fill (1) circle (2pt) node [below left] {H};
\fill (2) circle (2pt) node [below right] {H};
\fill (3) circle (2pt) node [above right] {H};
\fill (4) circle (2pt) node [above left] {H};
\node at (5) [below] {$1.0$ Bohr};

\draw[dashed] (1)--(2)--(3)--(4)-- cycle;

\end{tikzpicture}
\caption{Atomic configuration for the square H$_4$ model.}
\label{fig:h4conf}
\end{figure}

\begin{figure}
\centering
\includegraphics[width=0.6\textwidth]{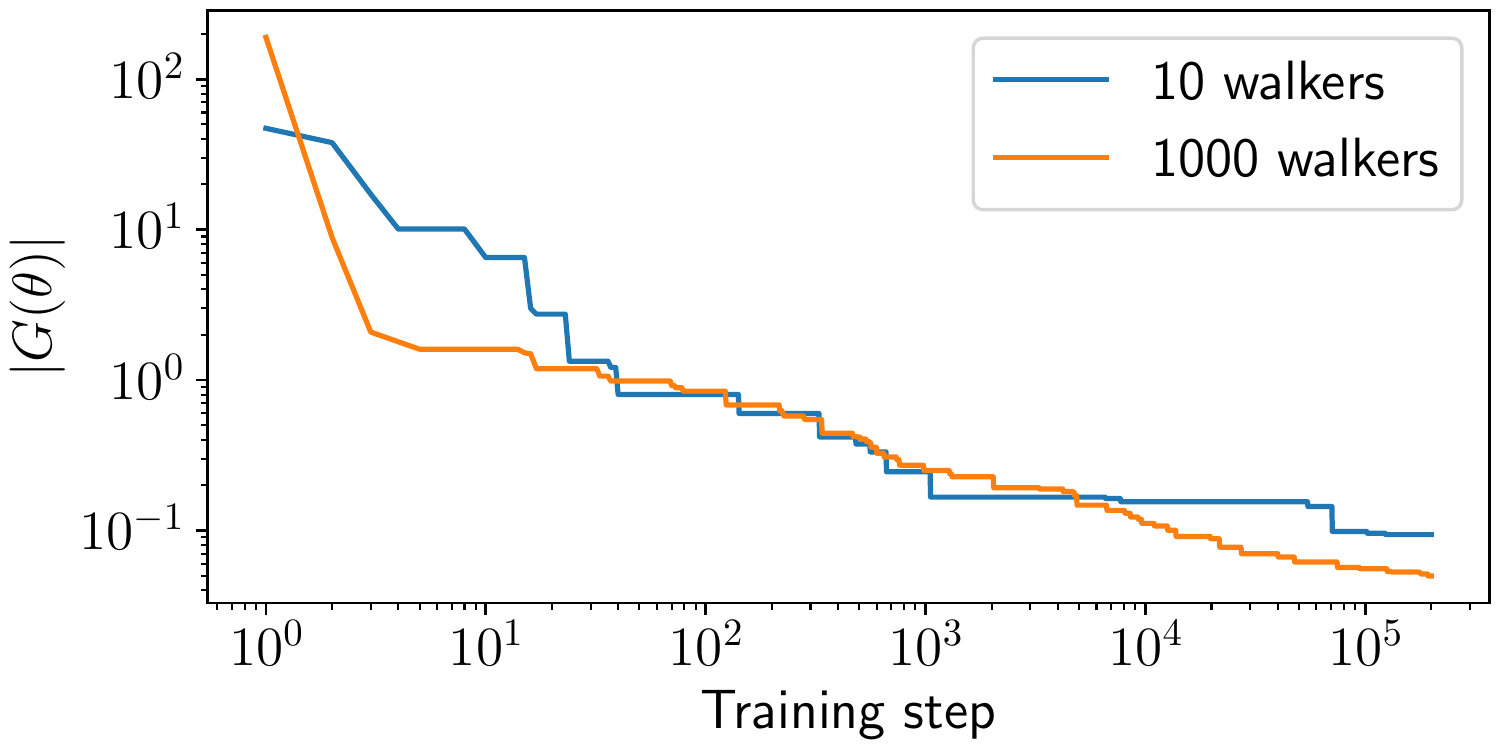}
\caption{Convergence of VMC run on the H$_4$ square. The running minimum is taken to smooth out the data and match the form of \cref{cor:VMC}.}
\label{fig:H4_G} 
\end{figure}

With this learning rate schedule, \cref{cor:VMC} implies that the running minimum of $\mathbb{E}(|\nabla_\theta L(\theta_m)|)$ should converge as $O(M^{-1/4})$ as the optimization progresses. To verify this bound, we plot in \cref{fig:H4_G} the running minimum of $|G_m|$, which is our numerical proxy for $|\nabla_\theta L(\theta)|$ as per \cref{lem:unbiased_gradient}. We compare two distinct VMC runs using 10 and 1000 walkers, respectively, to explore how the convergence varies with the accuracy of the gradient estimate. We estimate the convergence rate by measuring the overall slope of the log-log plot, starting from step 200 to avoid the initial pre-asymptotic period. We find that the gradient of the loss converges roughly as $M^{-0.27}$ when using 10 walkers, closely matching the expected bound. With 1000 walkers the convergence goes as $M^{-0.38}$, which is somewhat faster than the theoretical bound, as might be expected given that the estimate of the gradient is very accurate with so many walkers.
With 1000 walkers, the convergence is $M^{-0.38}$. This rate surpasses the theoretical bound given in Corollary \ref{cor:VMC}. Although a rigorous theoretical understanding of this phenomenon is currently lacking, we hypothesize that with this number of walkers, relatively accurate gradient estimates lead to training dynamics similar to (non-stochastic) gradient descent, giving a faster convergence rate than our theoretical bound in \cref{cor:VMC}. 

Since the bounding of the Lipschitz constant of the gradient of the loss function is critical to our convergence proof, we supply in \cref{fig:H4_L} a numerical estimate of this quantity during the VMC run. Here we show data only for the case of 1000 walkers, as the Lipschitz constant estimate is extremely noisy with 10 walkers. The data suggest that the value of the Lipschitz constant is well below $1000$ for this simulation.

\begin{figure}
\centering
\includegraphics[width=0.6 \textwidth]{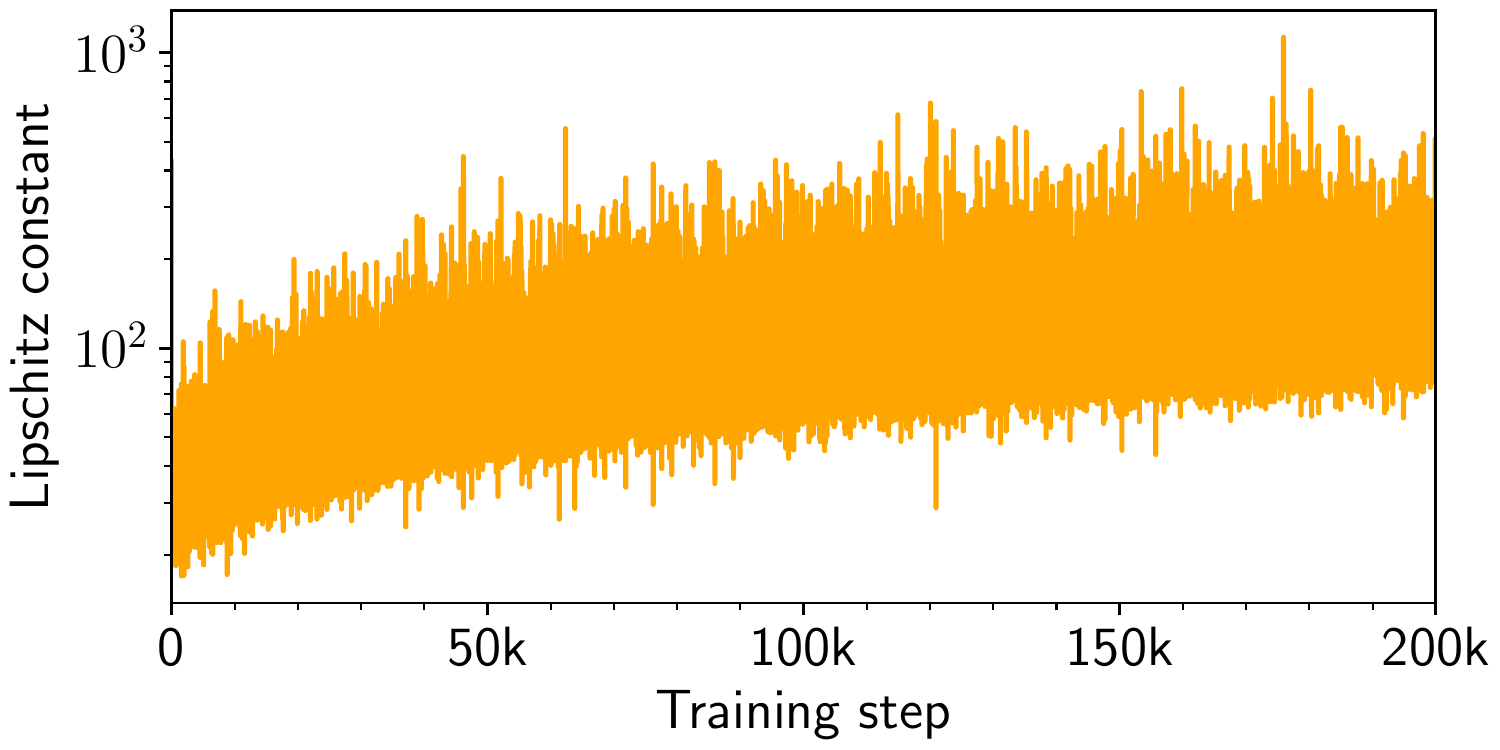}
\caption{Lipschitz constant for VMC run on the H$_4$ square with 1000 walkers. The constant is numerically approximated using the formula $|G(\theta_{m+1})- G(\theta_m)|/|\theta_{m+1} - \theta_m|$.
}
\label{fig:H4_L}
\end{figure}

\section{Conclusion}

We consider two optimization problems that arise in neural network variational Monte Carlo simulations of electronic structure: energy minimization and supervised pre-training. We provide theoretical convergence bounds for both cases. 
For the setting of supervised pre-training, we note that the standard algorithms do not incorporate the property of scale-invariance. We propose using a scale-invariant loss function for the supervised pre-training and demonstrate numerically that incorporating scale-invariance accelerates the process of pre-training. 

SGD is only the simplest stochastic optimization method. Over the last two decades, there has been increasing interest in developing more efficient and scalable optimization methods for VMC simulations~\cite{Sandvik2007, Neuscamman2012, Otis2019}, and our analysis may be a starting point for analyzing such methods.  It may also be possible to generalize this work from a high-dimensional sphere to a Grassmann manifold, parameterized by a neural network up to a gauge matrix. This setting could be applicable to VMC simulations of \textit{excited states} of quantum systems.

\section*{Acknowledgement}

This work was supported by the Simons Foundation under Award No. 825053 (N. A.), and by the Challenge Institute for Quantum Computation (CIQC) funded by National Science Foundation (NSF) through grant number OMA-2016245 (Z. D.). This work is supported by the U.S. Department of Energy, Office of Science, Office of Advanced Scientific Computing Research, Department of Energy Computational Science Graduate Fellowship under Award Number DE-SC0023112 (G. G.).
This material is also based upon work supported by the U.S. Department of Energy, Office of Science, Office of Advanced Scientific Computing Research and Office of Basic Energy Sciences, Scientific Discovery through Advanced Computing (SciDAC) program under Award Number DE-SC0022364 (L.L.).  L.L. is a Simons Investigator in Mathematics.  This research used the Savio computational cluster resource provided by the Berkeley Research Computing program at the University of California, Berkeley (supported by the UC Berkeley Chancellor, Vice Chancellor for Research, and Chief Information Officer).

\section*{Disclaimer}

This report was prepared as an account of work sponsored by an agency of the United States Government. Neither the United States Government nor any agency thereof, nor any of their employees, makes any warranty, express or implied, or assumes any legal liability or responsibility for the accuracy, completeness, or usefulness of any information, apparatus, product, or process disclosed, or represents that its use would not infringe privately owned rights. Reference herein to any specific commercial product, process, or service by trade name, trademark, manufacturer, or otherwise does not necessarily constitute or imply its endorsement, recommendation, or favoring by the United States Government or any agency thereof. The views and opinions of authors expressed herein do not necessarily state or reflect those of the United States Government or any agency thereof.

\bibliography{references}
\bibliographystyle{elsarticle-num}

\newpage
\appendix
\onecolumn

\section{Proofs of gradient estimators}
\label{gradproofs}
\begin{proof}[Proof of \eqref{eqn:gradient_L_VMC}]

Let $q_\theta(x)=|\psi_\theta(x)|^2$, and $Q_\theta=\|\psi_\theta\|^2$.

We write \eqref{eqn:VMC_L} as
\[L(\theta)=\ifrac{\ol{q_\theta}{\localE_\theta}}{Q_\theta}\,.\]
Then by the product formula,
\begin{equation}
\begin{aligned}
\nabla_\theta L(\theta)
&=
\frac{\nabla_\theta\ol{q_\theta}{\lE_\theta}}{Q_\theta}-\frac{\ol{q_\theta}{\lE_\theta}\nabla_\theta Q_\theta}{Q^2_\theta}=:A-B.
\end{aligned}
\label{AB}
\end{equation}
Continuing, we have
\begin{equation}
\label{eqA}
\begin{aligned}
A
&=\ol{\localE_\theta}{\nabla_\theta q_\theta}/Q_\theta+\ol{q_\theta}{\nabla_\theta\localE_\theta}/Q_\theta
\\&= \ol{\lE_\theta}{\frac{\nabla_\theta q_\theta}{q_\theta}\frac{q_\theta}{Q_\theta}}+\ol{\frac{q_\theta}{Q_\theta}}{\nabla_\theta\lE_\theta}
\\&= \EE_{p_\theta}[\lE_\theta\nabla_\theta\log q_\theta]+\EE_{p_\theta}[\nabla_\theta\lE_\theta].
\end{aligned}
\end{equation}
and
\begin{equation}
\label{eqB}
\begin{aligned}
B
&= \ol{\frac{q_\theta}{Q_\theta}}{\lE_\theta}\frac{\nabla_\theta Q_\theta}{Q_\theta}
\\&= L(\theta)\ol{\frac{q_\theta}{Q_\theta}}{\frac{\nabla_\theta q_\theta}{q_\theta}}
= L(\theta)\EE_{p_\theta}[\nabla_\theta\log q_\theta].
\end{aligned}
\end{equation}
Substitute \eqref{eqA}, \eqref{eqB} into \eqref{AB} to obtain 
\begin{equation}\label{eqn:gradient_L_VMC_true}
\begin{aligned}
\nabla_\theta L(\theta)
&=\EE_{X\sim p_\theta}\left[(\lE_\theta(X)-L(\theta))\nabla_\theta\log q_\theta(X)+\nabla_\theta\lE_\theta(X)\right]
\\&=
2\EE_{X\sim p_\theta}[\big(\lE_\theta(X)-L(\theta))\nabla_\theta\log|\psi_\theta(X)|+\nabla_\theta\lE_\theta(X)\Big]
\,.
\end{aligned}
\end{equation}

It remains to show that $\EE_{X\sim p_\theta}[\partial_i\lE_\theta(X)]=0$ where $\partial_i$ is differentiation with respect to $\theta_i$. Write
\begin{equation}
\partial_i\mathcal{E}_{\theta}=\partial_i\frac{H \psi_{\theta}}{\psi_{\theta}}=\frac{\partial_i H \psi_{\theta}\cdot\psi_{\theta}-H \psi_{\theta}\cdot \partial_i\psi_{\theta}}{\psi_{\theta}^{2}}\,.
\end{equation}
So, because $H$ is symmetric:
\[
\mathbb{E}_{X\sim p_\theta}\left[\partial_i\mathcal{E}_{\theta}\right]=\frac{\ol{H\partial_i\psi_\theta}{\psi_\theta}-\ol{H\psi_\theta}{\partial_i\psi_\theta}}{\|\psi_\theta\|^2}=0,
\]
\end{proof}

\begin{proof}[Proof of Lemma \ref{lem:unbiased_gradient}]
\[
\begin{aligned}
\Gtx  =&\frac{2}{n-1} \sum_{i=1}^{n} \mathcal{E}_{\theta}\left(X_{i}\right) \nabla_{\theta} \log|\psi_{\theta}\left(X_{i}\right)|\\
&-\frac{2}{n(n-1)} \sum_{j=1}^{n} \sum_{i=1}^{n} \mathcal{E}_{\theta}\left(X_{j}\right) \nabla_{\theta} \log|\psi_{\theta}\left(X_{i}\right)|\\
 =&\frac{2}{n} \sum_{i=1}^{n} \mathcal{E}_{\theta}\left(X_{i}\right) \nabla_{\theta} \log|\psi_{\theta}\left(X_{i}\right)|\\
 &-\frac{2}{n^{2}-n} \sum_{i \neq j} \mathcal{E}_{\theta}\left(X_{i}\right) \nabla_{\theta} \log|\psi_{\theta}\left(X_{j}\right)| .
\end{aligned}
\]
Here, the terms $i=j$ in the expansion of the rightmost term cancel with the first sum. But now each sum is the average of terms with the correct expectation $\left(2 \mathbb{E}_{X\sim p_\theta}\left[\mathcal{E}_{\theta}(X) \nabla_{\theta} \log|\psi_{\theta}(X)|\right]\right.$ and $2 \mathcal{L}_{\theta} \mathbb{E}_{X\sim p_\theta} \nabla_{\theta} \log|\psi_{\theta}(X)|$ respectively $)$.
\end{proof}

\section{The directionally unbiased gradient estimator for supervised learning}
\label{sec:minibath_sup}
 Given $n\ge2$ i.i.d. samples $\{X_i\}\sim\rho$. We write $\psi_i=\psi_\theta(X_i)$, $\varphi_i=\varphi(X_i)$, and $\nabla \psi_i=\nabla_\theta \psi_\theta(X_i)$.
 \begin{equation}
 \begin{aligned}
 \ol{\psi}{\varphi}_n=\frac1n\sum_{i=1}^n\psi_i \varphi_i,
 \qquad
 \|\psi\|_n^2=\frac1n\sum_{i=1}^n\psi_i^2.
 \end{aligned}
 \end{equation}
 \begin{lemma}\label{unbiasedsupervised}
 Given samples $X_1,\ldots,X_n$ from $\rho$, let
 \begin{align}
 G_n&=\frac1{n-1}\sum_{j=1}^n a_j\nabla \psi_j,\label{eqn:G_n}
 \\
 a_j&=-\|\psi\|_n^2\:\varphi_j+\ol{\varphi}{\psi}_n\:\psi_j.
 \end{align}
 Then $G_n$ is an unbiased estimator for $G=\|\psi_\theta\|^3_\rho\nabla_\theta L(\theta)$.
 \end{lemma}
 \begin{proof}[Proof of Lemma \ref{unbiasedsupervised}]
 We notice
 \[G=-\|\psi_\theta\|^2_\rho\ol{\varphi}{\nabla \psi_\theta}+\ol{\varphi}{\psi_\theta}\ol{\psi_\theta}{\nabla \psi_\theta}.\]
 For each $i\neq j$ we have an unbiased estimator for $G$ given by
 \[-|\psi_i|^2\varphi_j\nabla \psi_j+\varphi_i\psi_i\psi_j\nabla \psi_j.\]
 Taking the average over all pairs $i\neq j$ gives another unbiased estimator for $G$:
 \[-\frac1{n(n-1)}\sum_{i\neq j}(\psi_i^2\varphi_j\nabla \psi_j-\varphi_i\psi_i\psi_j\nabla \psi_j).\]
 We may add the terms $i=j$ without changing the value because all such terms evaluate to $0$.
 \end{proof}

\section{Proof of Theorem \ref{thm:complexity_VMC}}\label{sec:thm:complexity_VMC}
To prove Theorem \ref{thm:complexity_VMC}, we first show the following lemma:
\begin{lemma}\label{lem:prior_VMC}
Under \cref{assum_g_VMC} there exists a uniform constant $C'$ such that for any $\theta,\widetilde{\theta}\in\mathbb{R}^d$,
\begin{equation}\label{eqn:gradient_bound_vmc}
\left|\nabla_\theta L(\theta)\right|\leq 4C^2_\wf\,,
\end{equation}
\begin{equation}\label{eqn:gradient_Lip_vmc}
\left|\nabla_\theta L(\theta)-\nabla_\theta L\left(\widetilde{\theta}\right)\right|\leq C'(C^4_\wf+1)|\theta-\widetilde{\theta}|\,,
\end{equation}
and
\begin{equation}\label{eqn:variance_bound_vmc}
\mathbb{E}_{\{X_i\}^n_{i=1}}\left(\left|\Gtx\right|^2\right)\leq |\nabla_\theta L(\theta)|^2+\frac{C'(C^4_\wf+1)}{n}\,,
\end{equation}
where $\Gtx$ is defined in \eqref{eqn:VMCgradient_estimator}.
\end{lemma}
The above lemma is important for the proof of Theorem \ref{thm:complexity_VMC}. First, inequality \eqref{eqn:variance_bound_vmc} gives an upper bound for the gradient estimator. Using this upper bound, we can show that each iteration of Algorithm \ref{algo:VMC} is close to the classical gradient descent when the learning rate $\eta$ is small enough. Second, \eqref{eqn:gradient_Lip_vmc} means that the hessian of $L$ is bounded. This ensures that the classical gradient descent has a faster convergence rate to the first-order stationary point, which further implies the fast convergence of Algorithm \ref{algo:VMC}.

\begin{proof}[Proof of Lemma \ref{lem:prior_VMC}] Define
\[
Z_\theta=\sqrt{\int_{\Omega}\left|\psi_{\theta}(x)\right|^{2} dx}\,.
\]

First, to prove \eqref{eqn:gradient_bound_vmc}, using \eqref{eqn:gradient_L_VMC}, we have
\[
\begin{aligned}
\left|\nabla_\theta L(\theta)\right|&=2\mathbb{E}_{X\sim p_\theta}\left|\left(\mathcal{E}_{\theta}(X)-L(\theta)\right) \frac{\nabla_{\theta} \psi_{\theta}(X)}{\psi_{\theta}(X)}\right|
\\
&\leq 2\left(\mathbb{E}_{X\sim p_\theta}
\left|\mathcal{E}_{\theta}(X)-L(\theta)\right|^2\right)^{1/2}
\left(
\mathbb{E}_{X\sim p_\theta}
\left|
\frac{\nabla_{\theta} \psi_{\theta}(X)}{\psi_{\theta}(X)}
\right|^2
\right)^{1/2}
\leq 4C^{2}_\psi\,,
\end{aligned}
\]
where we use H\"older's inequality in the first inequality, \eqref{eqn_value_gnbound_vmc} and the second inequality of  \eqref{eqn_Gradient_gnbound_vmc} in the last inequality.

Next, to prove \eqref{eqn:gradient_Lip_vmc}. Using the formula of $\mathcal{E}(\theta)$ (equation \eqref{eqn:deflocalE}), we write 
\[
\begin{aligned}
\left\|\mathsf{H}_{\theta}L(\theta)\right\|=&\underbrace{\frac{2\left|\nabla_\theta Z_\theta\right|}{Z^3_\theta}\int_{\Omega} \left|H\psi_\theta(x)\nabla_\theta \psi_\theta(x)\right|dx}_{\mathrm{(I)}}+\underbrace{\frac{1}{Z^2_\theta}\int_{\Omega} \left\|\nabla_\theta H\psi_\theta(x)\nabla_\theta \psi_\theta(x)\right\|dx}_{\mathrm{(II)}}\\
&+\underbrace{\frac{1}{Z^2_\theta}\int_{\Omega} \left|H\psi_\theta(x)\right|\left\|\mathsf{H}_\theta \psi_\theta(x)\right\|dx}_{\mathrm{(III)}}+\left\|\nabla_\theta \left(\frac{1}{Z^2_\theta}\int_{\Omega} L(\theta)\psi_\theta(x)\nabla_\theta \psi_\theta(x)dx\right)\right\|\,.
\end{aligned}
\]
We first deal with Terms (I), (II), (III) separately:
\begin{itemize}
\item For Term (I), we notice
\[
\left|\nabla_\theta Z_\theta\right|=\left|\frac{\int \nabla_\theta \psi_\theta(x)\psi_\theta(x)dx}{\sqrt{\int |\psi_\theta(x)|^2dx}}\right|\leq \sqrt{\int |\nabla_\theta \psi_\theta(x)|^2dx}\,,
\]
where we use H\"older's inequality to bound the numerator in the inequality. This implies
\[
\begin{aligned}
&\mathrm{Term\ (I)}\leq \frac{2\sqrt{\int |\nabla_\theta \psi_\theta(x)|^2dx}}{\sqrt{\int | \psi_\theta(x)|^2dx}}\int_{\Omega} \frac{\left|H\psi_\theta(x)\nabla_\theta \psi_\theta(x)\right|}{Z^2_\theta}dx\\
\leq &2\left(\mathbb{E}_{X\sim p_\theta}\left(\left|\frac{\nabla_\theta  \psi_\theta(X)}{\psi_\theta(X)}\right|^2\right)\right)^{1/2}\frac{\sqrt{\int_\Omega|H\psi_\theta(x)|^2dx}\sqrt{\int_\Omega|\nabla_\theta\psi_\theta(x)|^2dx}}{Z^2_\theta}\\
=&2\left(\mathbb{E}_{X\sim p_\theta}\left(\left|\frac{\nabla_\theta  \psi_\theta(X)}{\psi_\theta(X)}\right|^2\right)\right)^{1/2}\left(\mathbb{E}_{X\sim p_\theta}\left(\left|\frac{H\psi_\theta(X)}{\psi_\theta(X)}\right|^2\right)\right)^{1/2}\left(\mathbb{E}_{X\sim p_\theta}\left(\left|\frac{\nabla_\theta  \psi_\theta(X)}{\psi_\theta(X)}\right|^2\right)\right)^{1/2}\\
\leq & 2C^3_\psi\,.
\end{aligned}
\]
Here, we use the H\"older's inequality in the second inequality, and \eqref{eqn_value_gnbound_vmc}-\eqref{eqn_Hessian_gnbound_vmc} in the last inequality.

\item For Term (II),
    \begin{align*}
    &\frac{1}{Z^2_\theta}\int_{\Omega} \left\|\nabla_\theta H\psi_\theta(x)\nabla_\theta \psi_\theta(x)\right\|dx\leq \mathbb{E}_{X\sim p_{\theta}}\left(\frac{\left|\nabla_\theta H\psi_{\theta}(X)\right|\left|\nabla_{\theta} \psi_{\theta}(X)\right|}{\left|\psi_{\theta}(X)\right|^2}\right)\\
    \leq & \left(\mathbb{E}_{X\sim p_\theta}\left(\left|\frac{\nabla_\theta H\psi_\theta(X)}{\psi_\theta(X)}\right|^2\right)\right)^{1/2}\left(\mathbb{E}_{X\sim p_\theta}\left(\left|\frac{\nabla_\theta  \psi_\theta(X)}{\psi_\theta(X)}\right|^2\right)\right)^{1/2}\leq C^2_\psi
    \end{align*}
where we use H\"older's inequality, \eqref{eqn_value_gnbound_vmc}-\eqref{eqn_Hessian_gnbound_vmc} in the last two inequalities.

\item For Term (III), 
    \begin{align*}
    &\frac{1}{Z^2_\theta}\int_{\Omega} \left\|H\psi_\theta(x)\mathsf{H}_\theta \psi_\theta(x)\right\|dx\leq \mathbb{E}_{X\sim p_{\theta}}\left(\frac{\left| H\psi_{\theta}(X)\right|\left\|\mathsf{H}_\theta \psi_\theta(x)\right\|}{\left|\psi_{\theta}(X)\right|^2}\right)\\
    \leq & \left(\mathbb{E}_{X\sim p_\theta}\left(\left|\frac{H\psi_\theta(X)}{\psi_\theta(X)}\right|^2\right)\right)^{1/2}\left(\mathbb{E}_{X\sim p_\theta}\left(\left|\frac{\mathsf{H}_\theta \psi_\theta(x)}{\psi_\theta(X)}\right|^2\right)\right)^{1/2}\leq C^2_\psi
    \end{align*}
where we use H\"older's inequality, \eqref{eqn_value_gnbound_vmc}-\eqref{eqn_Hessian_gnbound_vmc} in the last two inequalities.
\end{itemize}
Combining the above three inequalities, we have
\[
\text{Term (I)+Term (II)+Term (III)}\leq 2C^2_\psi(C_\psi+1)\,.
\]
Using a similar calculation, we can also show
\[
\left\|\nabla_\theta \left(\frac{1}{Z^2_\theta}\int_{\Omega} L(\theta)\psi_\theta(x)\nabla_\theta \psi_\theta(x)dx\right)\right\|\leq C(C^4_\wf+1)\,,
\]
where $C$ is a uniform constant. Thus, we have the hessian of $L$ can be bounded, meaning
\[
\left\|\mathsf{H}_{\theta}L(\theta)\right\|\leq C(C^4_\wf+1)\,,
\]
which proves \eqref{eqn:gradient_Lip_vmc} by the mean-value theorem.

Finally, to prove \eqref{eqn:variance_bound_vmc}, noticing
\[
\Gtx=\frac{2}{n}\sum_{i=1}^{n} \mathcal{E}_{\theta}\left(X_{i}\right) \nabla_{\theta} \log |\psi_{\theta}\left(X_{i}\right)|-\frac{2}{n^{2}-n} \sum_{i \neq j} \mathcal{E}_{\theta}\left(X_{i}\right) \nabla_{\theta} \log |\psi_{\theta}\left(X_{j}\right)|\,,
\]
and
\[
\mathbb{E}_{\{X_i\}^n_{i=1}}\left(\Gtx\right)=\nabla_\theta L(\theta)\,,
\]
we have
\[
\begin{aligned}
  &\mathbb{E}_{\{X_i\}^n_{i=1}}\left(\left|\Gtx\right|^2\right)\\
  \leq &|\nabla_\theta L(\theta)|^2+\frac{8}{n^2}\sum^n_{i=1}\mathbb{E}_{X_i\sim p_\theta}\left(\left|\mathcal{E}_{\theta}\left(X_{i}\right) \nabla_{\theta} \log |\psi_{\theta}\left(X_{i}\right)|-\mathbb{E}_{X\sim p_\theta}\left[\mathcal{E}_{\theta}(X) \nabla_{\theta} \log |\psi_{\theta}(X)|\right]\right|^2\right)\\
  &+\frac{8}{(n^2-n)^2}\mathbb{E}\left(\sum_{i\neq j}\mathcal{E}_{\theta}\left(X_{i_1}\right) \nabla_{\theta} \log |\psi_{\theta}\left(X_{j_1}\right)|-\mathcal{L}_{\theta} \mathbb{E}_{X\sim p_\theta} \nabla_{\theta} \log |\psi_{\theta}(X)|\right)^2\\
  \leq &|\nabla_\theta L(\theta)|^2+\frac{8}{n}\mathbb{E}_{X\sim p_\theta}\left|\mathcal{E}_{\theta}\left(X\right) \nabla_{\theta} \log |\psi_{\theta}\left(X\right)|-\mathbb{E}_{X\sim p_\theta}\left[\mathcal{E}_{\theta}(X) \nabla_{\theta} \log |\psi_{\theta}(X)|\right]\right|^2\\
  &+\frac{8}{(n^2-n)^2}\left(\sum_{i_1=i_2\,\text{or}\,j_1=j_2} 1\right)\mathbb{E}_{(X_1,X_2)\sim p^{\otimes 2}_\theta}\left|\mathcal{E}_{\theta}\left(X_{1}\right) \nabla_{\theta} \log |\psi_{\theta}\left(X_{2}\right)|- \mathcal{L}_{\theta} \mathbb{E}_{X\sim p_\theta} \nabla_{\theta} \log |\psi_{\theta}(X)|\right|^2\\
  \leq &|\nabla_\theta L(\theta)|^2+\frac{8}{n}\mathbb{E}_{X\sim p_\theta}\left|\mathcal{E}_{\theta}\left(X\right) \nabla_{\theta} \log |\psi_{\theta}\left(X\right)|\right|^2+\frac{C}{n}\mathbb{E}_{(X_1,X_2)\sim p^{\otimes 2}_\theta}\left|\mathcal{E}_{\theta}\left(X_{1}\right) \nabla_{\theta} \log |\psi_{\theta}\left(X_{2}\right)|\right|^2\\
\end{aligned}\,,
\]
where we use $X_{i}$ and $X_j$ are independent in the first two inequalities. 
Using \eqref{eqn_value_gnbound_vmc} and the second inequality of  \eqref{eqn_Gradient_gnbound_vmc}, it is straightforward to show:
\[
\mathbb{E}_{X\sim p_\theta}\left|\mathcal{E}_{\theta}\left(X\right) \nabla_{\theta} \log |\psi_{\theta}\left(X\right)|\right|^2\leq \left(\mathbb{E}_{X\sim p_\theta}\left(\left|\frac{H \psi_\theta(X)}{\psi_\theta(X)}\right|^4\right)\right)^{1/2}\left(\mathbb{E}_{X\sim p_\theta}\left(\left|\frac{\nabla_\theta \psi_\theta(X)}{\psi_\theta(X)}\right|^4\right)\right)^{1/2}\leq C^4_\wf\,,
\]
and
\[
\mathbb{E}_{(X_1,X_2)\sim p^{\otimes 2}_\theta}\left|\mathcal{E}_{\theta}\left(X_{1}\right) \nabla_{\theta} \log |\psi_{\theta}\left(X_{2}\right)|\right|^2\leq \mathbb{E}_{X\sim p_\theta}\left(\left|\frac{H \psi_\theta(X)}{\psi_\theta(X)}\right|^2\right)\mathbb{E}_{X\sim p_\theta}\left(\left|\frac{\nabla_\theta \psi_\theta(X)}{\psi_\theta(X)}\right|^2\right)\leq C^4_\wf\,.
\]
This concludes the proof of \eqref{eqn:variance_bound_vmc}.
\end{proof}


Now, we are ready to prove Theorem \ref{thm:complexity_VMC}:
\begin{proof}[Proof of Theorem \ref{thm:complexity_VMC}]
Denote the probability filtration 
\[
\mathcal{F}_m=\sigma(X^j_i,1\leq i\leq n, 1\leq j\leq m)\,.
\]
According to the algorithm, we have
\[
\theta_{m+1}=\theta_m-\eta_m G_m\,.
\]
Plugging $\theta_{m+1}$ into $L(\theta)$, we have
\[
L(\theta_{m+1})\leq L(\theta_m)-\eta_m\nabla_\theta L(\theta_m)\cdot G_m+\frac{C\eta_m^2}{2}|G_m|^2
\]
where we use \eqref{eqn:gradient_Lip_vmc} and $C$ is a constant that only depends on $C_\wf$. This implies
\[
\begin{aligned}
\mathbb{E}\left(L(\theta_{m+1})|\mathcal{F}_{m-1}\right)\leq &L(\theta_m)-\eta_m\left|\nabla_\theta L(\theta_m)\right|^2+\frac{C\eta_m^2}{2}\mathbb{E}(|G_m|^2|\mathcal{F}_{m-1})
\end{aligned}
\]
where we use $\mathbb{E}_{\{X_i\}^n_{i=1}}\left(\Gtx\right)=\nabla_\theta L(\theta)$. Furthermore, using \eqref{eqn:variance_bound_vmc} and $\eta_m<\frac{1}{C}$, we have 
\[
\begin{aligned}
\mathbb{E}\left(L(\theta_{m+1})|\mathcal{F}_{m-1}\right)\leq &L(\theta_m)-\left(\eta_m-C\eta_m^2/2\right)\left|\nabla_\theta L(\theta_m)\right|^2+\frac{C\eta_m^2}{2n}\\
\leq &L(\theta_m)-\frac{\eta_m}{2}\left|\nabla_\theta L(\theta_m)\right|^2+\frac{C\eta_m^2}{2n}
\end{aligned}
\]
Taking the average in $m=1,2,3,\cdots,M-1$, we prove \eqref{eqn:gradient_bound_vmc_1}.
\end{proof}

\section{Proof of Theorem \ref{thm:main_result_alg1}}\label{sec:thm:main_result_alg1}
In this section, we prove Theorem \ref{thm:main_result_alg1}. We first show the following lemma: 
\begin{lemma}\label{lem:svsl_2} Under \cref{assum_g_scale_inva} there exists a constant $C'$ that only depends on $C_\psi,C_r,C_\phi$ such that for any $\theta,\widetilde{\theta}\in\mathbb{R}^d$,
\begin{equation}\label{eqn:gradient_Lip_2}
\left|\nabla_\theta L(\theta)-\nabla_\theta L\left(\widetilde{\theta}\right)\right|\leq C'|\theta-\widetilde{\theta}|\,.
\end{equation}
and
\begin{equation}\label{eqn:variance_bound_sup}
\mathbb{E}_{\{X_i\}^n_{i=1}}\left(\left|\Gtx\right|^2\right)\leq \frac{Z^6_\theta}{Z^6}\left(|\nabla_\theta L(\theta)|^2+\frac{C'}{n}\right)\,,
\end{equation}
where $\Gtx$ is defined in \eqref{eqn:G_r_theta}.
\end{lemma}
\begin{proof}[Proof of Lemma \ref{lem:svsl_2}] The proof is very similar to the proof of Lemma \ref{lem:prior_VMC} after setting $H\psi_\theta(x)=g(x)\left(\functional{g(x)}{\psi_\theta(x)}\right)$.
\end{proof}

Now, we are ready to prove Theorem \ref{thm:main_result_alg1}.
\begin{proof}[Proof of Theorem \ref{thm:main_result_alg1}] 
According to the algorithm, we have
\[
\theta_{m+1}=\theta_m-\eta_m G_m\,.
\]
Plugging $\theta_{m+1}$ into $L(\theta)$, we have
\[
L(\theta_{m+1})\leq L(\theta_m)-\eta_m\nabla_\theta L(\theta_m)\cdot G_m+\frac{C\eta_m^2}{2}|G_m|^2
\]
where we use \eqref{eqn:gradient_Lip_2}. Here $C$ is a constant that only depends on $C_\psi$. Because
\begin{equation}\label{eqn:expectation_G}
\mathbb{E}_{X^m_1,\cdots,X^m_n}(G_m)=\frac{Z^3_m}{\widetilde{Z}^3_m}\nabla_\theta L(\theta_m)\,,
\end{equation}
we obtain
\begin{equation}\label{eqn:L_bound}
\begin{aligned}
\mathbb{E}\left(L(\theta_{m+1})|\mathcal{L}_{m-1}\right)\leq &L(\theta_m)-\eta_m\left|\nabla_\theta L(\theta_m)\right|^2\frac{Z^3_m}{\widetilde{Z}^3_m}
\\
&+\frac{C\eta^2_m}{2}\mathbb{E}(|G_m|^2|\mathcal{L}_{m-1})
\end{aligned}
\end{equation}

Using \eqref{eqn:variance_bound_sup} and \eqref{eqn:distance_bound},
\begin{equation}\label{eqn:bound_fourth_term}
\mathbb{E}(|G_m|^2|\mathcal{L}_{m-1})\leq C\left(|\nabla_\theta L(\theta_m)|^2+\frac{C'}{n}\right)\,,
\end{equation}
where $C$ is a constant depends on $C_\psi,C_r,C_\phi$.
 
Plugging \eqref{eqn:bound_fourth_term} into \eqref{eqn:L_bound}, using \eqref{eqn:distance_bound}, and choosing $\eta_m$ small enough, we have
\begin{equation}\label{eqn:bound_psi_2}
\mathbb{E}\left(L(\theta_{m+1})\right)\leq L(\theta_m)-\frac{\eta_m}{2C^2_r}\left|\nabla_\theta L(\theta_m)\right|^2+\frac{C\eta^2_m}{n}\\
\end{equation}
Taking the average in $m=1,2,3,\cdots,M-1$, we prove \eqref{eqn:gradient_psi_bound}.
\end{proof}

\section{Lipschitz of the $Z_\theta$}\label{sec:svsl_1}
\begin{lemma}\label{lem:svsl_1} Assume \cref{assum_g_scale_inva}, let $\theta,\widetilde{\theta}\in\mathbb{R}^d$, and suppose 
\[
\left|\widetilde{\theta}-\theta\right|\leq C_r\,.
\]
Then there exists a constant $C'$ that only depends on $C_\psi,C_r,C_\phi$ such that
\begin{equation}\label{eqn:ratio_Lip_1}
\frac{\|\psi_{\widetilde{\theta}}\|_\rho}{\|\psi_{\theta}\|_\rho}\leq \exp(C'C_r)\,.
\end{equation}
\end{lemma}

\begin{proof}[Proof of Lemma \ref{lem:svsl_1}] Fixing $\theta,\widetilde{\theta}$, we define
\[
R(t)=\frac{\left\|\psi_{\theta+t(\widetilde{\theta}-\theta)}\right\|^2_\rho}{\|\psi_{\theta}\|^2_\rho}\,.
\]
Then
\[
\begin{aligned}
|R'(t)|\leq &\left|\frac{2\functional{\nabla_\theta \psi_{\theta+t(\widetilde{\theta}-\theta)}}{\psi_{\theta+t(\widetilde{\theta}-\theta)}}_\rho}{\|\psi_{\theta}\|^2_\rho}\right||\widetilde{\theta}-\theta|\\
\leq &\left|\frac{2\functional{\nabla_\theta \psi_{\theta+t(\widetilde{\theta}-\theta)}}{\psi_{\theta+t(\widetilde{\theta}-\theta)}}_\rho}{\left\|\psi_{\theta+t(\widetilde{\theta}-\theta)}\right\|^2_\rho}\right|\left|\frac{\left\|\psi_{\theta+t(\widetilde{\theta}-\theta)}\right\|^2_\rho}{\|\psi_{\theta}\|^2_\rho}\right||\widetilde{\theta}-\theta|\\
\leq &C|\widetilde{\theta}-\theta|R(t)\,,
\end{aligned}
\]
where we use \eqref{eqn_value_gnbound} and \eqref{eqn_Gradient_gnbound} in the last inequality. Because $R(0)$=1, using Gr\"onwall's inequality, we obtain
\[
\frac{\|\psi_{\widetilde{\theta}}\|_\rho}{\|\psi_{\theta}\|_\rho}=R^{1/2}(1)\leq \exp(C|\theta-\widetilde{\theta}|)\leq \exp(CC_r)\,.
\]
\end{proof}

\section{Additional details on scale-invariant supervised training}
\label{sec:appnumerics}

We created a fork \cite{ferminet_github_fork} of the FermiNet repository \cite{ferminet_github} and implemented the scale-invariant loss function \cref{SI_pretrain} in the columns of the backflow tensors. We used the straight-through gradient estimator \cite{yin2019understanding} to train the network using either the standard loss or our modified scale-invariant loss while plotting the angle between the wave functions in each case \cref{fig:comp}. 

The commands for each (using the fork \cite{ferminet_github_fork} of FermiNet) were: 

\begin{itemize}
\item Standard (blue)

\begin{verbatim}
ferminet --config ferminet/configs/atom.py \
--config.system.atom Li \
--config.batch_size 256 \
--config.pretrain.iterations 10 
\end{verbatim}

\item Scale-invariant (orange)

\begin{verbatim}
ferminet --config ferminet/configs/atom.py \
--config.system.atom Li \
--config.batch_size 256 \
--config.pretrain.iterations 10 \
--config.pretrain.SI   
\end{verbatim}

\end{itemize}

\end{document}